\documentclass[letterpaper, 10 pt, conference]{ieeeconf}  %

\IEEEoverridecommandlockouts

\usepackage[linesnumbered,vlined,ruled,commentsnumbered]{algorithm2e}
\SetVlineSkip{0.2em}%
\DontPrintSemicolon%
\SetKwIF{If}{ElseIf}{Else}{if}{}{else if}{else}{endif}%
\SetAlgoSkip{}%
\SetAlCapHSkip{0cm}%
\usepackage{amsmath}
\usepackage{mathtools}
\usepackage{amssymb,bm}
\usepackage{enumerate}
\usepackage{float}
\usepackage{tikz,pgfplots}
\usetikzlibrary{shapes.misc}
\usetikzlibrary{external}
\tikzset{>=latex}
\tikzset{vertex/.style = {circle, inner sep=0.5pt, fill=muibluegrey, draw=muibluegrey}}
\tikzset{cross/.style = {cross out, draw=muibluegrey, thick, fill=none, minimum size=2 * (#1-\pgflinewidth), inner sep=0pt, outer sep=0pt}, cross/.default={2pt}}
\usepackage{subfig}
\usepackage{nicefrac}

\makeatletter
\let\NAT@parse\relax
\makeatother
\usepackage[numbers,sort&compress]{natbib}

\usepackage{amsthm}
\newtheoremstyle{mpsdefinition}{2pt plus 2pt minus 0pt}{2pt plus 2pt minus 0pt}
{\itshape}
{0pt}{\bfseries}{.}{ }{\thmname{#1}\thmnumber{ #2}\thmnote{ \normalfont(#3)}}%
\theoremstyle{mpsdefinition}

\theoremstyle{plain}
\newtheorem{theorem}{Theorem}

\newcommand*{\defeq}{\mathrel{\vcenter{\baselineskip0.5ex \lineskiplimit0pt
                     \hbox{\scriptsize.}\hbox{\scriptsize.}}}%
                 =}

\newcommand{\mpsexpand}[1]{\texttt{expand}\left({#1}\right)}
\newcommand{\mpsupdateapproximation}[1]{\texttt{update\_approximation}\left(#1\right)}
\newcommand{\mpsissearchmarkedfinished}[1]{\texttt{is\_search\_marked\_finished}\left(#1\right)}
\newcommand{\mpsmarksearchfinished}[1]{\texttt{mark\_search\_finished}\left(#1\right)}
\newcommand{\mpsmarksearchunfinished}[1]{\texttt{mark\_search\_unfinished}\left(#1\right)}
\newcommand{\mpsprune}[1]{\texttt{prune}\left(#1\right)}

\newcommand{\mpssample}[1]{\texttt{sample}\left(#1\right)}
\newcommand{\mpsupdateinflationfactor}[1]{\texttt{update\_inflation\_factor}\left(#1\right)}
\newcommand{\mpsupdatetruncationfactor}[1]{\texttt{update\_truncation\_factor}\left(#1\right)}

\usepackage{xcolor}

\definecolor{muired}{HTML}{e53935}
\definecolor{muipink}{HTML}{d81b60}
\definecolor{muipurple}{HTML}{8e24aa}
\definecolor{muideeppurple}{HTML}{5e35b1}
\definecolor{muiindigo}{HTML}{3949ab}
\definecolor{muiblue}{HTML}{1e88e5}
\definecolor{muilightblue}{HTML}{039be5}
\definecolor{muicyan}{HTML}{00acc1}
\definecolor{muiteal}{HTML}{00897b}
\definecolor{muigreen}{HTML}{43a047}
\definecolor{muilightgreen}{HTML}{7cb342}
\definecolor{muilime}{HTML}{c0ca33}
\definecolor{muiyellow}{HTML}{fdd835}
\definecolor{muiamber}{HTML}{ffb300}
\definecolor{muiorange}{HTML}{fb8c00}
\definecolor{muideeporange}{HTML}{f4511e}
\definecolor{muibrown}{HTML}{6d4c41}
\definecolor{muigrey}{HTML}{757575}
\definecolor{muibluegrey}{HTML}{546e7a}

\usepackage[hidelinks]{hyperref}

\title{\LARGE \bf
Advanced BIT* (ABIT*):\\ Sampling-Based Planning with Advanced Graph-Search Techniques
}

\author{Marlin P.\ Strub$^{1}$ and Jonathan D.\ Gammell$^{1}$
\thanks{$^{1}$M.\ P.\ Strub and J.\ D.\ Gammell are with the Estimation, Search, and Planning (ESP) Group of the Oxford Robotics Institute (ORI), University of Oxford, United Kingdom. {\tt\footnotesize (mstrub|gammell)@robots.ox.ac.uk}}%
}

\begin{document}

\maketitle
\thispagestyle{empty}
\pagestyle{empty}

\begin{abstract}

  Path planning is an active area of research essential for many applications in robotics. Popular techniques include graph-based searches and sampling-based planners. These approaches are powerful but have limitations.

  This paper continues work to combine their strengths and mitigate their limitations using a unified planning paradigm. It does this by viewing the path planning problem as the two subproblems of search and approximation and using advanced graph-search techniques on a sampling-based approximation.

  This perspective leads to Advanced BIT*. ABIT* combines truncated anytime graph-based searches, such as ATD*, with anytime almost-surely asymptotically optimal sampling-based planners, such as RRT*. This allows it to quickly find initial solutions and then converge towards the optimum in an anytime manner. ABIT* outperforms existing single-query, sampling-based planners on the tested problems in \( \mathbb{R}^{4} \) and \( \mathbb{R}^{8} \), and was demonstrated on real-world problems with NASA/JPL-Caltech.\

\end{abstract}


\bstctlcite{IEEEexample:BSTcontrol} 

\section{Introduction}%
\label{sec:introduction}

Popular path planning algorithms in robotics include graph-based searches, such as A$^{*}$~\citep{hart1968} and Dijkstra's~\citep{dijkstra1959}, and sampling-based planners, such as Rapidly-exploring Random Trees (RRT)~\citep{lavalle2001} and Probabilistic Roadmaps (PRM)~\citep{kavraki1996}. Both graph- and sampling-based approaches have characteristic strengths and limitations. Previous work~\citep{gammell2015,gammell2020} has separated search and approximation in single-query, almost-surely asymptotically optimal sampling-based planning to combine their strengths and mitigate their limitations. This separation can be leveraged to use advanced graph-based search techniques on an anytime sampling-based approximation to further improve performance.

An important strength of graph-based approaches is their strong theoretical guarantees. A* is \emph{resolution-optimal} (and \emph{resolution-complete}) as well as \emph{optimally efficient}. Any other algorithm guaranteed to find the optimal solution must expand at least as many vertices as A*, given the same problem information~\citep{hart1968}. This efficiency is achieved by always expanding the state with the highest potential solution quality but requires an ordering of the search that does not provide any solutions until the optimum is found.

\emph{Anytime} graph-search algorithms sacrifice the efficiency of A* to find intermediate solutions faster. Anytime Repairing A* (ARA*)~\citep{likhachev2004} finds an initial, potentially suboptimal solution quickly and then uses any remaining computational time to improve it. ARA* does this without duplicated search effort by starting with a heuristic that is inflated and then gradually decreasing this inflation while accounting for changes to state connections (i.e., inconsistent vertices).

\begin{figure}[t]
  \centering
  \includegraphics[width = \columnwidth]{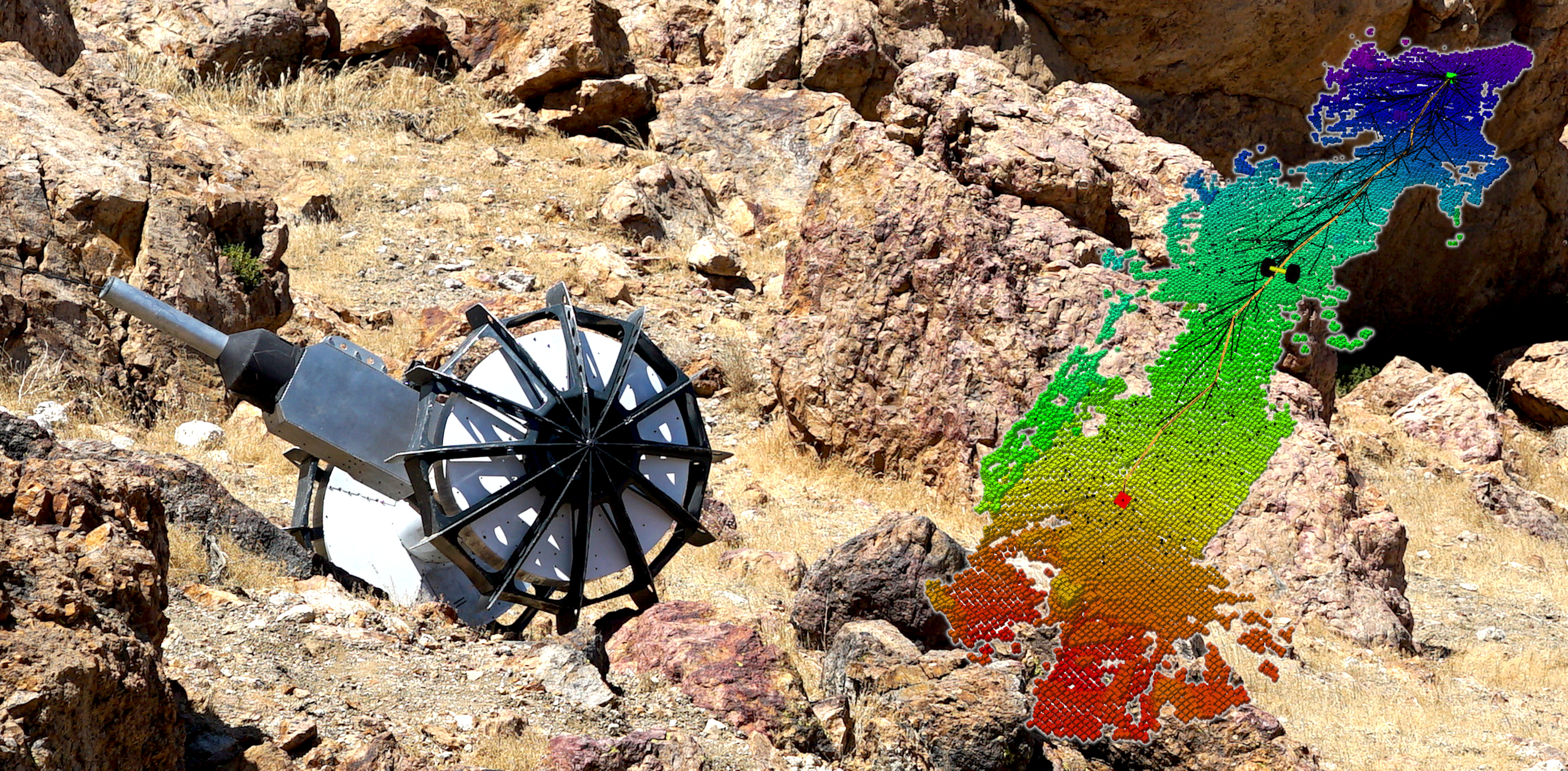}
  \caption{\footnotesize ABIT* on NASA/JPL-Caltech's Axel Rover System~\protect\citep{nesnas2012} during a week-long field test in the Mojave Desert, California, USA (Section~\protect\ref{sec:planning-for-axel}).
  }%
\label{fig:teaser}
\end{figure}

These graph-based algorithms cannot be applied directly to planning problems with continuously valued state spaces and selecting appropriate \textit{a priori} discretizations is difficult. Coarse resolutions are computationally inexpensive to search but may yield paths of poor quality in a continuous context. Fine resolutions contain paths of higher quality~\citep{bertsekas1975} but the associated computational cost becomes prohibitively expensive in high dimensions (i.e., \textit{the curse of dimensionality}~\citep{bellman1957}).

Sampling-based planners sample states incrementally to avoid picking a resolution \textit{a priori}. Single-query planners such as RRT simultaneously build and search an anytime approximation that improves with computational time. This incremental sampling makes the search dependent on the sampling order and results in a randomly ordered search.

Ordering the search of an incremental planner in a solution-oriented manner requires either ordering the approximation or separating the search from the approximation. Batch Informed Trees (BIT*)~\citep{gammell2015,gammell2020} achieves the second by sampling batches of states and viewing them as an increasingly dense implicit random geometric graph (RGG)~\citep{penrose2003}. BIT* uses incremental search techniques, similar to Lifelong Planning A* (LPA*)~\citep{koenig2004}, to process the sampled states in order of potential solution quality.

This paper extends BIT* to further leverage the separation of search and approximation in sampling-based planning. Advanced BIT* (ABIT*) uses advanced graph-search techniques, such as inflation and truncation, to balance exploring and exploiting an increasingly dense RGG approximation. ABIT* is almost-surely asymptotically optimal and outperforms existing single-query, almost-surely asymptotically optimal planners on random and artificially designed problems, especially in high dimensions. The benefits of ABIT* were shown on real-world problems on NASA/JPL-Caltech's Axel~(Fig.~\ref{fig:teaser}) during a week-long field test in the Mojave Desert testing autonomous navigation on challenging terrain.

\section{Background}%
\label{sec:background}

Multiquery algorithms, such as PRM and PRM*~\citep{karaman2011}, separate the approximation and the search in two phases. These algorithms first sample the entire collision-free state space to approximate it with a graph which is then searched to solve individual problems. This approximation is computationally expensive but justified by its reuse.

Lazy-PRM~\citep{bohlin2000} reduces the approximation cost by initially assuming that all vertices and edges are collision-free. A planning problem is then solved by searching this graph and checking the resulting path for collisions. If it contains collisions then the corresponding vertices and edges are removed and a new search is started on the updated graph. 

Single-query problems only require an approximation of the state space relevant to the query. Expansive Space Trees (EST)~\citep{hsu1997} avoids approximating the entire state space by incrementally growing trees from the start and goal. This focuses the approximation to the relevant region of the state space but couples the search with the approximation and results in a randomly ordered search.

RRT grows a single tree from the start to incrementally approximate the state space. It rapidly expands this tree by biasing its growth towards unexplored regions of the space (i.e., with \textit{Voronoi bias}). The approximation and search are still coupled by the incremental sampling which also results in a randomly ordered search.

RRT-Connect~\citep{kuffner2000} grows trees from the start and goal and biases their growth both towards the unexplored space and each other. The greedy tree connection often results in fast initial solution times but the search of the unexplored space remains randomly ordered by the incremental sampling.

These algorithms provide no solution quality guarantee. RRT*~\citep{karaman2011} extends RRT to provide \emph{almost-surely asymptotically optimal} solutions. Its solutions converge to the optimum with probability one given infinite computational time by locally rewiring the tree to ensure locally optimal connections. This search for the optimal local connection is performed immediately after each new state is sampled and the global search of the state space remains random.

Improving the convergence rate of RRT* is an active area of research. Informed RRT*~\citep{gammell2018} and RRT*-Smart~\citep{nasir2013} are based on different sampling strategies. RRT$^{\#}$~\citep{arslan2013} modifies the rewiring procedure to ensure globally optimal connections. Lower Bound Tree-RRT (LBT-RRT)~\citep{salzman2015} is a near-optimal variant that allows for continuous interpolation between RRT and RRT* by maintaining a tree which stores a lower bound on the cost-to-come to all vertices. These extensions modify how states are used but still process them individually and maintain the random search order of RRT*.

Approximation and search can be separated by processing batches of multiple states. Fast Marching Tree (FMT*)~\citep{janson2015} samples a user-specified number of states and uses a Fast Marching Method~\citep{valero2013} to grow a tree that connects them in order of increasing cost-to-come. FMT* is not anytime and must be restarted from scratch if no suitable solution is found with the specified number of samples.

BIT* samples multiple batches of states and views the resulting approximation as an increasingly dense edge-implicit RGG\@. It performs a search ordered on the potential solution quality of the implicit edges of this graph. It efficiently reuses previous search efforts by incorporating techniques from incremental graph-search approaches, similar to LPA*. BIT* is anytime but uses a nonanytime graph-search that only returns a single solution per approximation. This exploits the current approximation without considering that it will be updated once the search is finished, which may be inefficient.

Fast-BIT*~\citep{holston2017} is a variation of BIT* that finds initial solutions faster by ordering the initial search solely on cost-to-go. Once a solution is found, it reorders its queues with a full-solution heuristic and reprocesses all vertices. This maintains almost-sure asymptotic optimality but duplicates previous search effort unnecessarily. Fast-BIT* then performs the same incremental search as BIT* to fully exploit each subsequent approximation.

Unlike the multiquery planners, ABIT* builds a problem-specific approximation of the state space. Unlike EST and the RRT-based approaches, ABIT* separates approximation and search in a way that allows it to process states in order of (inflated) potential solution cost. Unlike FMT*, ABIT* improves its approximation in an anytime fashion. Unlike BIT* and Fast-BIT*, ABIT* avoids wasting computational effort to find a resolution-optimal path in an inaccurate approximation that will change.

\section{Advanced Batch Informed Trees (ABIT*)}%
\label{sec:advanced-batch-informed-trees}

BIT* is an anytime, single-query planner that almost-surely asymptotically finds an optimal solution to a (continuously valued) planning problem. It focuses its approximation of the state space to the region that can improve the current solution using informed sampling~\citep{gammell2018}. This approximation is separated from the search by sampling batches of states and viewing them as an increasingly dense edge-implicit RGG.\

This perspective enables BIT* to perform a series of informed graph-searches in which RGG edges are processed in order of their potential solution quality. This is achieved by sorting an edge queue according to the sum of the current cost-to-come from the start to the edge's parent state, an estimate of the edge cost, and an estimate of the cost-to-go from the edge's child state to a goal. BIT* performs these searches efficiently by reusing information from both previous searches and approximations similar to incremental search algorithms, e.g. LPA*. Full details are in~\citep{gammell2017,gammell2020}.

BIT*'s separation of approximation and search provides a direction for better single-query, almost-surely asymptotically optimal planning algorithms. ABIT* builds on this separation by using more advanced graph-search techniques. It accelerates anytime performance without duplicating search effort, similar to anytime repairing graph-search algorithms, and avoids wasting effort to find the resolution-optimal solution in an approximation that will change, similar to truncated incremental graph-search algorithms.

This accelerated performance is achieved by inflating the cost-to-go estimate in ABIT*'s edge queue. This sacrifices resolution-optimality but achieves faster initial solution times compared to the incremental search used by BIT*. ABIT*'s initial (potentially suboptimal) solution is subsequently repaired without duplicating search effort by tracking inconsistent states, similar to ARA*. A state is considered inconsistent if its cost-to-come has decreased since its outgoing edges were last inserted into the queue.

Fully exploiting every RGG approximation by finding the resolution-optimal path is computationally expensive. ABIT* avoids this by truncating its search as soon as it can guarantee that it has found a solution whose cost is within a factor of the resolution-optimal cost. This allows ABIT* to balance the exploitation of its approximation (i.e., repairing the search) with the exploration of the state space (i.e., increasing the density of the approximation).

\subsection{Notation}%
\label{sec:notation}

The state space of the planning problem is denoted by \( X \), the start state by \( \bm{\mathrm{x}}_{\mathrm{start}} \in X \), and the goal states by \( X_{\mathrm{goal}} \subset X \). The current search is stored as a tree, \( \mathcal{T} = (V, E) \), with vertices, \( V \), and edges, \( E \subset V \times V \). Vertices in the tree are associated with valid states and edges in the tree represent valid connections between states. An edge consists of a parent state, \( \bm{\mathrm{x}}_{\mathrm{p}} \), and a child state, \( \bm{\mathrm{x}}_{\mathrm{c}} \), and is denoted as \( (\bm{\mathrm{x}}_{\mathrm{p}}, \bm{\mathrm{x}}_{\mathrm{c}}) \). The set of inconsistent states is denoted by \( V_{\mathrm{inconsistent}} \) and the states that are not in the tree make up the set \( X_{\mathrm{unconnected}} \).

Let \( \mathbb{R}_{\geq 0}^{\infty} \) denote the union of all nonnegative real numbers with infinity. The function \( \widehat{g}\colon X \to \mathbb{R}_{\geq 0}^{\infty} \) represents an admissible estimate (i.e., a lower bound) of the cost-to-come from the start to a state, \( \bm{\mathrm{x}} \in X \). The function \( g_{\mathcal{T}}\colon X \to \mathbb{R}_{\geq 0}^{\infty} \) represents the cost-to-come from the start to a state, \( \bm{\mathrm{x}} \in X \), through the current search tree, \( \mathcal{T} \). This cost is taken to be infinite for any state with no associated vertex in the tree.

The function \( \widehat{h}\colon X \to \mathbb{R}_{\geq 0}^{\infty} \) represents an admissible estimate of the cost-to-go from a state to a goal. The function \( \widehat{f}\colon X \to \mathbb{R}_{\geq 0}^{\infty} \) represents an admissible estimate of the cost of a path from the start to a goal constrained to go through a state, e.g., \( \widehat{f}(\bm{\mathrm{x}}) \defeq \widehat{g}(\bm{\mathrm{x}}) + \widehat{h}(\bm{\mathrm{x}}) \). This estimate defines the informed set of states that could provide a better solution, \( X_{\widehat{f}} \defeq \{ \bm{\mathrm{x}} \in X \;|\; \widehat{f}(\bm{\mathrm{x}}) \leq c_{\mathrm{current}} \} \), where \( c_{\mathrm{current}} \) is the current solution cost~\citep{gammell2018}. The function \( c\colon X \times X \to \mathbb{R}_{\geq 0}^{\infty} \) denotes the true edge cost between two states. The function \( \widehat{c}\colon X \times X \to \mathbb{R}_{\geq 0}^{\infty} \) is an admissible estimate of this cost.

Let \( A \) be any set and let \( B \) and \( C \) be subsets of \( A \), i.e., \( B, C \subseteq A \). The notation \( B \overset{+}{\leftarrow} C \) is used for \( B \leftarrow B \cup C \) and \( B \overset{-}{\leftarrow} C \) for \( B \leftarrow B \setminus C \). The cardinality of a set is denoted by \( | \cdot | \) and the minimum of an empty set is taken to be infinity. The Lebesgue measure of a set is denoted by \( \lambda(\cdot) \), and the Lebesgue measure of an \( n \)-dimensional unit ball by \( \zeta_{n} \). The number of states per batch is denoted by \( m \).

\newlength{\textfloatsepbackup}
\setlength{\textfloatsepbackup}{\textfloatsep}
\setlength{\textfloatsep}{0pt}

\begin{algorithm}[t]
\footnotesize
\SetInd{0.0em}{0.4em}
\caption{\texttt{ABIT}${}^*(\bm{\mathrm{x}}_{\text{start}}, X_{\text{goal}}, m)$}%
\label{alg:main}
$V \leftarrow \{ \bm{\mathrm{x}}_{\mathrm{start}} \}; E \leftarrow \emptyset; \mathcal{T} \leftarrow (V, E); X_{\mathrm{unconnected}} \leftarrow X_{\mathrm{goal}}$\;\label{line:main:init:start}
$q \leftarrow \left| V \right| + \left| X_{\mathrm{unconnected}} \right|; \varepsilon_{\mathrm{infl}} \leftarrow \infty; \varepsilon_{\mathrm{trunc}} \leftarrow \infty$\;\label{line:main:init:r}
$V_{\mathrm{closed}} \leftarrow \emptyset; V_{\mathrm{inconsistent}} \leftarrow \emptyset$\;\label{line:main:init:closed}
$\mathcal{Q} \leftarrow \mpsexpand{\{\bm{\mathrm{x}}_{\mathrm{start}}\}, \mathcal{T}, X_{\mathrm{unconnected}}, \infty}$\;\label{line:main:init:end}
\Repeat{{\normalfont$\mathtt{stop}$}}{
  \eIf{\normalfont$\mpsissearchmarkedfinished{}$}{\label{line:main:exhausted:search}
    \eIf{\normalfont$\mpsupdateapproximation{\varepsilon_{\mathrm{infl}}, \varepsilon_{\mathrm{trunc}}}$}{\label{line:main:approximation:start}\label{line:main:exhausted:approximation}
      $\mpsprune{\mathcal{T}, X_{\mathrm{unconnected}}, X_{\mathrm{goal}}}$\;\label{line:main:approximation:prune}
      $X_{\mathrm{unconnected}} \overset{+}{\leftarrow} \mpssample{m, X_{\mathrm{goal}}}$\;
      $q \leftarrow \left| V \right| + \left| X_{\mathrm{unconnected}} \right|$\;
      $\mathcal{Q} \leftarrow \mpsexpand{\{\bm{\mathrm{x}}_{\mathrm{start}}\}, \mathcal{T}, X_{\mathrm{unconnected}}, r(q)}$\;
    }{
      $\mathcal{Q} \overset{+}{\leftarrow} \mpsexpand{V_{\mathrm{inconsistent}}, \mathcal{T}, X_{\mathrm{unconnected}}, r(q)}$\;
    }\label{line:main:approximation:end}
    $\varepsilon_{\mathrm{infl}} \leftarrow \mpsupdateinflationfactor{}$\;\label{line:main:update:inflation}
    $\varepsilon_{\mathrm{trunc}} \leftarrow \mpsupdatetruncationfactor{}$\;\label{line:main:update:truncation}
    $V_{\mathrm{closed}} \leftarrow \emptyset$; $V_{\mathrm{inconsistent}} \leftarrow \emptyset$\;
    $\mpsmarksearchunfinished{}$\;
  }{
    \vspace*{-0.3em}
    $(\bm{\mathrm{x}}_{\mathrm{p}}, \bm{\mathrm{x}}_{\mathrm{c}}) \leftarrow \mathop{\arg\min}\limits_{(\bm{\mathrm{x}}_{i}, \bm{\mathrm{x}}_{j}) \in \mathcal{Q}}  \left\{ g_{\mathcal{T}} (\bm{\mathrm{x}}_{i}) + \widehat{c}(\bm{\mathrm{x}}_{i}, \bm{\mathrm{x}}_{j}) + \varepsilon_{\mathrm{infl}}\widehat{h}(\bm{\mathrm{x}}_{j}) \right\}$\;\label{line:main:search:start}\label{line:main:search:queue}

    $\mathcal{Q} \overset{-}{\leftarrow} (\bm{\mathrm{x}}_{\mathrm{p}}, \bm{\mathrm{x}}_{\mathrm{c}}) $\;
    \uIf{\normalfont$(\bm{\mathrm{x}}_{\mathrm{p}}, \bm{\mathrm{x}}_{\mathrm{c}}) \in E$}{
      \eIf{\normalfont$\bm{\mathrm{x}}_{\mathrm{c}} \in V_{\mathrm{closed}}$}{\label{line:main:search:incons:freebie:start}
        $V_{\mathrm{inconsistent}} \overset{+}{\leftarrow} \bm{\mathrm{x}}_{\mathrm{c}}$\;
      }{
        $\mathcal{Q} \overset{+}{\leftarrow} \mpsexpand{\{\bm{\mathrm{x}}_{\mathrm{c}}\}, \mathcal{T}, X_{\mathrm{unconnected}}, r(q)}$\;
        $V_{\mathrm{closed}} \overset{+}{\leftarrow} \bm{\mathrm{x}}_{\mathrm{c}}$\;\label{line:main:search:incons:freebie:end}
      }
    }
    \vspace*{-0.2em}
    \uElseIf{\normalfont$\varepsilon_{\mathrm{trunc}}\hspace*{-0.2em}\left(g_{\mathcal{T}}(\bm{\mathrm{x}}_{\mathrm{p}}) + \widehat{c}(\bm{\mathrm{x}}_{\mathrm{p}}, \bm{\mathrm{x}}_{\mathrm{c}}) + \widehat{h}(\bm{\mathrm{x}}_{\mathrm{c}}) \right) \leq \min\limits_{\mathclap{\bm{\mathrm{x}} \in X_{\mathrm{goal}}}} \left\{ g_{\mathcal{T}}(\bm{\mathrm{x}}) \right\}$}{\label{line:main:checks:first}
      \vspace*{-0.2em}
      \If{\normalfont$g_{\mathcal{T}}(\bm{\mathrm{x}}_{\mathrm{p}}) + \widehat{c}(\bm{\mathrm{x}}_{\mathrm{p}}, \bm{\mathrm{x}}_{\mathrm{c}}) < g_{\mathcal{T}}(\bm{\mathrm{x}}_{\mathrm{c}})$}{\label{line:main:checks:second}
        \If{\normalfont$g_{\mathcal{T}}(\bm{\mathrm{x}}_{\mathrm{p}}) + c(\bm{\mathrm{x}}_{\mathrm{p}}, \bm{\mathrm{x}}_{\mathrm{c}}) + \widehat{h}(\bm{\mathrm{x}}_{\mathrm{c}}) < \min\limits_{\mathclap{\bm{\mathrm{x}} \in X_{\mathrm{goal}}}} \left\{ g_{\mathcal{T}}(\bm{\mathrm{x}}) \right\}$}{\label{line:main:checks:third}
          \If{\normalfont$g_{\mathcal{T}}(\bm{\mathrm{x}}_{\mathrm{p}}) + c(\bm{\mathrm{x}}_{\mathrm{p}}, \bm{\mathrm{x}}_{\mathrm{c}}) < g_{\mathcal{T}}(\bm{\mathrm{x}}_{\mathrm{c}})$}{\label{line:main:checks:fourth}
            \eIf{\normalfont$\bm{\mathrm{x}}_{\mathrm{c}} \in V$}{
              $E \overset{-}{\leftarrow} \{ (\bm{\mathrm{x}}_{\mathrm{prev}}, \bm{\mathrm{x}}_{\mathrm{c}}) \in E \}$\;\label{line:main:rewiring}
            }{
              $X_{\mathrm{unconnected}} \overset{-}{\leftarrow} \bm{\mathrm{x}}_{\mathrm{c}}$\;\label{line:main:remove:samples}
              $V \overset{+}{\leftarrow} \bm{\mathrm{x}}_{\mathrm{c}}$\;\label{line:main:add:vertex}
            }
            $E \overset{+}{\leftarrow} (\bm{\mathrm{x}}_{\mathrm{p}}, \bm{\mathrm{x}}_{\mathrm{c}})$\;\label{line:main:add:edge}
            \eIf{\normalfont$\bm{\mathrm{x}}_{\mathrm{c}} \in V_{\mathrm{closed}}$}{\label{line:main:search:incons:start}
              $V_{\mathrm{inconsistent}} \overset{+}{\leftarrow} \bm{\mathrm{x}}_{\mathrm{c}}$\;
            }{
              $\mathcal{Q} \overset{+}{\leftarrow} \mpsexpand{\{\bm{\mathrm{x}}_{\mathrm{c}}\}, \mathcal{T}, X_{\mathrm{unconnected}}, r(q)}$\;
              $V_{\mathrm{closed}} \overset{+}{\leftarrow} \bm{\mathrm{x}}_{\mathrm{c}}$\;\label{line:main:search:end}\label{line:main:search:incons:end}
            }
          }
        }
      }
    }
    \lElse{$\mpsmarksearchfinished{}$\label{line:main:mark}}
  }
}
\end{algorithm}

\setlength{\textfloatsep}{\textfloatsepbackup}

\subsection{Initialization (Algorithm~\ref{alg:main}, Lines~\ref{line:main:init:start}-\ref{line:main:init:end})}%
\label{sec:initialization}

ABIT* starts by initializing the search tree with the start state as its root. The set of unconnected states initially only contains the goal states (line~\ref{line:main:init:start}). The inflation factor, \( \varepsilon_{\mathrm{infl}} \geq 1 \), and the truncation factor, \( \varepsilon_{\mathrm{trunc}} \geq 1 \), are initialized to be infinitely large (line~\ref{line:main:init:r}). The edge-queue, \( \mathcal{Q} \), holds all edges from the start state to the goal states (line~\ref{line:main:init:end}).

\subsection{Approximation (Algorithm~\ref{alg:main}, Lines~\ref{line:main:approximation:start}-\ref{line:main:approximation:end})}%
\label{sec:approximation}

ABIT* uses informed sampling~\citep{gammell2018} to focus its RGG approximation on the relevant region of the state space. The accuracy of this approximation increases with the number of sampled states but so does its complexity. This complexity is reduced by pruning states that cannot improve the current solution (line~\ref{line:main:approximation:prune} and Alg.~\ref{alg:prune}) and shrinking the connection radius as more states are sampled. The radius, \( r \), is updated as in~\citep{karaman2011}, using the measure of the informed set, as in~\citep{gammell2018},%
\vspace*{-0.2em}%
\begin{align*}
  r(q) \defeq \eta {\left( 2 \left( 1 + \frac{1}{n} \right) \left(\frac{\lambda\left( X_{\widehat{f}} \right)}{\zeta_{n}}\right) \left(\frac{\log(q)}{q}\right) \right)}^{\frac{1}{n}},
\end{align*}
where \( q \) is the number of sampled states in the informed set, \( \eta > 1 \) is a tuning parameter, and \( n \) is the state space dimension. Faster-decreasing radii are provided in~\citep{janson2015,janson2018} but are not used in this paper to isolate the reasons for ABIT*'s improved performance relative to existing algorithms.

\subsection{Search (Algorithm~\ref{alg:main}, Lines~\ref{line:main:search:start}-\ref{line:main:search:end})}%
\label{sec:search}

ABIT* delays expensive computation of true edge cost (e.g., collision checks) with a lazy search similar to an edge-queue version of Anytime Truncated D* (ATD*)~\citep{aine2016b}. This queue is ordered lexicographically by (inflated) potential solution cost and then cost-to-come. A search iteration starts by removing the edge with the lowest queue value from the queue (line~\ref{line:main:search:start}). If this edge is part of the search tree, then the child state is expanded (i.e., its outgoing edges are added to the queue) if it has not already been expanded during the current search (lines~\ref{line:main:search:incons:freebie:start}--\ref{line:main:search:incons:freebie:end} and Alg.~\ref{alg:expand}). ABIT* otherwise checks if the new edge can possibly contribute to a solution better than the current one (lines~\ref{line:main:checks:first}--\ref{line:main:checks:fourth}).

An edge that passes these checks improves the cost-to-come of the child state and possibly the current solution. If the child state is already part of the tree, adding this edge constitutes a rewiring (line~\ref{line:main:rewiring}). Otherwise, this state is removed from the set of unconnected states (line~\ref{line:main:remove:samples}) and added to the search tree (line~\ref{line:main:add:vertex}). In both cases, the edge is added to the search tree (line~\ref{line:main:add:edge}).

After adding an edge, the child state is expanded unless it has already been expanded during the current search, in which case it is added to the set of inconsistent vertices (lines~\ref{line:main:search:incons:start}--\ref{line:main:search:incons:end}).

\begin{algorithm}[t]
\scriptsize
\SetInd{0.0em}{0.4em}
\caption{\texttt{expand}$(\{ \bm{\mathrm{x}}_{i} \}, \mathcal{T}, X_{\mathrm{unconnected}}, r)$}%
\label{alg:expand}
$E_{\mathrm{out}} \leftarrow \emptyset$\;
\ForAll{\normalfont$\bm{\mathrm{x}}_{\mathrm{p}} \;\bm{\mathrm{in}}\; \{ \bm{\mathrm{x}}_{i} \}$}{
  $E_{\mathrm{out}} \overset{+}{\leftarrow} \{ (\bm{\mathrm{x}}_{\mathrm{p}}, \bm{\mathrm{x}}_{\mathrm{c}}) \in E \}$\;
  \ForAll{\normalfont$\bm{\mathrm{x}}_{\mathrm{c}} \;\bm{\mathrm{in}}\; \big\{ \bm{\mathrm{x}} \in \{ X_{\mathrm{unconnected}} \cup V \} \;\big|\; \| \bm{\mathrm{x}}_{\mathrm{p}} - \bm{\mathrm{x}} \| \leq r \big\}$}{
    \If{\normalfont$\widehat{g}(\bm{\mathrm{x}}_{\mathrm{p}}) + \widehat{c}(\bm{\mathrm{x}}_{\mathrm{p}}, \bm{\mathrm{x}}_{\mathrm{c}}) + \widehat{h}(\bm{\mathrm{x}}_{\mathrm{c}}) \leq \min\limits_{\mathclap{\bm{\mathrm{x}} \in X_{\mathrm{goal}}}}\left\{g_{\mathcal{T}}(\bm{\mathrm{x}})\right\}$}{
      \If{\normalfont$\widehat{g}(\bm{\mathrm{x}}_{\mathrm{p}}) + \widehat{c}(\bm{\mathrm{x}}_{\mathrm{p}}, \bm{\mathrm{x}}_{\mathrm{c}}) \leq g_{\mathcal{T}}(\bm{\mathrm{x}}_{\mathrm{c}})$}{
        $E_{\mathrm{out}} \overset{+}{\leftarrow} (\bm{\mathrm{x}}_{\mathrm{p}}, \bm{\mathrm{x}}_{\mathrm{c}})$\;
      }
    }
  }
}
\Return{\normalfont$E_{\mathrm{out}}$}
\end{algorithm}
\begin{algorithm}[t]
\scriptsize
\caption{$\mpsprune{\mathcal{T}, X_{\mathrm{unconnected}}, X_{\mathrm{goal}}}$}%
\label{alg:prune}
$X_{\mathrm{unconnected}} \overset{-}{\leftarrow} \Big\{ \bm{\mathrm{x}} \in X_{\mathrm{unconnected}} \Big|\, \widehat{f}(\bm{\mathrm{x}}) \geq \min\limits_{\mathclap{\bm{\mathrm{x}} \in X_{\mathrm{goal}}}} \left\{ g_{\mathcal{T}}(\bm{\mathrm{x}}) \right\} \Big\}$\;
\vspace*{-0.1em}
$V \overset{-}{\leftarrow} \Big\{ \bm{\mathrm{x}} \in V \;\Big|\; \widehat{f}(\bm{\mathrm{x}}) > \min\limits_{\mathclap{\bm{\mathrm{x}} \in X_{\mathrm{goal}}}} \left\{ g_{\mathcal{T}}(\bm{\mathrm{x}}) \right\} \Big\}$\;
$E \overset{-}{\leftarrow} \Big\{ (\bm{\mathrm{x}}_{\mathrm{p}}, \bm{\mathrm{x}}_{\mathrm{c}}) \in E \Big|\, \widehat{f}(\bm{\mathrm{x}}_{\mathrm{p}}) > \min\limits_{\mathclap{\bm{\mathrm{x}} \in X_{\mathrm{goal}}}}\left\{ g_{\mathcal{T}}(\bm{\mathrm{x}}) \right\} \; \bm{\mathrm{or}}\, \widehat{f}(\bm{\mathrm{x}}_{\mathrm{c}}) > \min\limits_{\mathclap{\bm{\mathrm{x}} \in X_{\mathrm{goal}}}}\left\{ g_{\mathcal{T}}(\bm{\mathrm{x}}) \right\}\Big\}$\;
$X_{\mathrm{unconnected}} \overset{+}{\leftarrow} \{ \bm{\mathrm{x}}_{\mathrm{c}} \in V \;|\; \not\exists \; \bm{\mathrm{x}}_{\mathrm{p}} \in V \text{ s.\ t.\ } (\bm{\mathrm{x}}_{\mathrm{p}}, \bm{\mathrm{x}}_{\mathrm{c}}) \in E \}$\;
$V \overset{-}{\leftarrow} \{ \bm{\mathrm{x}}_{\mathrm{c}} \in V \;|\; \not\exists \; \bm{\mathrm{x}}_{\mathrm{p}} \in V \text{ s.\ t.\ } (\bm{\mathrm{x}}_{\mathrm{p}}, \bm{\mathrm{x}}_{\mathrm{c}}) \in E \}$\;
\end{algorithm}

\subsection{Approximation, Inflation, and Truncation Update Policies}%
\label{sec:infl-trunc-fact}

The approximation is updated when a desired bound on the resolution optimality is achieved, which depends on the inflation and truncation factors (line~\ref{line:main:exhausted:approximation}). These factors are updated after each search of the current RGG (lines~\ref{line:main:update:inflation} and~\ref{line:main:update:truncation}). A high inflation factor biases the search towards the goal and decreases solution times but results in loose bounds on the solution quality. A low inflation factor results in a search that requires more computational effort to complete but achieves tighter bounds on the solution quality. A high truncation factor promotes exploration of the region of the state space that could potentially contain better solutions by truncating the search once a loose bound on the solution quality is achieved, which facilitates adding more samples. A low truncation factor promotes exploiting the current approximation of the state space as the search is not truncated until a tight bound on the solution quality is guaranteed.

The update policies of these two factors are user-tuned parameters that balance exploiting the current RGG with exploring the state space. Section~\ref{sec:experiments} presents the specific policies used for the experimental results.

\section{Formal Analysis}%
\label{sec:seper-appr-search}

This paper uses Definition 24 in~\citep{karaman2011} as the definition of almost-sure asymptotic optimality. Note that any sampling-based planner is almost-surely asymptotically optimal if (i)~its underlying graph almost-surely contains an asymptotically optimal path, and (ii)~its underlying graph-search is asymptotically resolution-optimal. These conditions are sufficient but not necessary.

\subsection{Almost-Sure Existence of an Asymptotically Optimal Path}%
\label{sec:almost-sure-exist}

ABIT* uses the same increasingly dense RGG approximation as BIT*. Since BIT* is an almost-surely asymptotically optimal algorithm~\citep{gammell2020}, this approximation must almost-surely contain an asymptotically optimal path.

\subsection{Asymptotically Resolution-Optimal Search}%
\label{sec:resol-optim-guar}

Theorem~\ref{thm:asymptotically-optimal-search} states that ABIT*'s search processes at least all of the edges processed by ATD*, which is an anytime, incremental search algorithm that finds a solution within \( \varepsilon_{\mathrm{trunc}}\varepsilon_{\mathrm{infl}} \) of the optimum~\citep{aine2016b}. Since ABIT* updates the cost-to-come of any vertex under the same condition as ATD*, ABIT* also finds a solution whose cost is within \( \varepsilon_{\mathrm{trunc}}\varepsilon_{\mathrm{infl}} \) of the optimum. ABIT* therefore asymptotically finds a resolution-optimal path when the product \( \varepsilon_{\mathrm{trunc}}\varepsilon_{\mathrm{infl}} \) tends to one as the number of samples approaches infinity,
\begin{align*}
  \lim_{q \to \infty}\varepsilon_{\mathrm{trunc}}\varepsilon_{\mathrm{infl}} = 1.
\end{align*}

\begin{theorem}\label{thm:asymptotically-optimal-search}
  ABIT*'s search processes at least all of the edges that would be processed by ATD* for a given RGG.\
\end{theorem}

\begin{proof}
ATD* can handle improved and worsened state connections, but adding states and edges to a graph can only improve connections. Therefore only states with improved cost-to-come (called overconsistent in~\citep{aine2016b}) are considered.

ATD*'s vertex queue is first converted to a compatible edge queue and it is then shown that ATD*'s termination criterion is stricter than that of ABIT*.

ATD* computes a sort key for every state, \( \bm{\mathrm{x}} \), in its queue,
\begin{align*}
  \mathrm{key}_{\mathrm{ATD}^{*}}(\bm{\mathrm{x}}) \defeq g_{\mathrm{p}}[\bm{\mathrm{x}}] + \varepsilon_{\mathrm{infl}} \widehat{h}(\bm{\mathrm{x}}).
\end{align*}
The cost-to-come label, \( g_{\mathrm{p}}[\bm{\mathrm{x}}] \), is recursively defined as
\begin{align*}
 g_{\mathrm{p}}[\bm{\mathrm{x}}] \defeq \min_{\mathclap{\bm{\mathrm{x}}_{\mathrm{p}} \in X_{\mathrm{p}}(\bm{\mathrm{x}})}} \left\{ g_{\mathrm{p}}[\bm{\mathrm{x}}_{\mathrm{p}}] + c(\bm{\mathrm{x}}_{\mathrm{p}}, \bm{\mathrm{x}}) \right\},
\end{align*}
where \( X_{\mathrm{p}}(\bm{\mathrm{x}}) \) denotes the discovered potential parents of \( \bm{\mathrm{x}} \) and the base case is \( g_{\mathrm{p}}[\bm{\mathrm{x}}_{\mathrm{start}}] = 0 \). ATD* processes vertices in its vertex queue, \( \mathcal{Q}_{\mathrm{ATD}^{*}}^{\mathrm{V}} \), in order of ascending key values,%
\vspace*{-0.2em}%
\begin{align*}
  \bm{\mathrm{x}}_{\mathrm{next}} = \mathop{\arg\,\min}_{\bm{\mathrm{x}} \in \mathcal{Q}_{\mathrm{ATD}^{*}}^{\mathrm{V}}} \left\{ \mathrm{key}_{\mathrm{ATD}^{*}}(\bm{\mathrm{x}}) \right\}.
\end{align*}%
Whenever a better connection to a state in the queue is found, the key of this state is updated and the queue is resorted.

The queue could alternatively contain multiple instances of the same state, each with a different parent and key value. Selecting the minimum from this queue would ensure that the best discovered connection for each state is considered first. This would be equivalent to an edge version of ATD* where the next connection from the edge queue, \( \mathcal{Q}_{\mathrm{ATD}^{*}}^{\mathrm{E}} \), is%
\vspace*{-0.2em}%
\begin{align*}
  {(\bm{\mathrm{x}}_{\mathrm{p}}, \bm{\mathrm{x}}_{\mathrm{c}})}_{\mathrm{next}} = \hspace*{-1em}\mathop{\arg\,\min}_{(\bm{\mathrm{x}}_{i}, \bm{\mathrm{x}}_{j}) \in \mathcal{Q}_{\mathrm{ATD}^{*}}^{\mathrm{E}}} \hspace*{-1em} \left\{ g_{\mathrm{p}}[\bm{\mathrm{x}}_{i}] + c(\bm{\mathrm{x}}_{i}, \bm{\mathrm{x}}_{j}) + \varepsilon_{\mathrm{infl}}\widehat{h}(\bm{\mathrm{x}}_{j}) \right\}.
\end{align*}%
\vspace*{-0.5em}%

ATD*'s inner loop terminates if for any goal \( \bm{\mathrm{x}}_{\mathrm{goal}} \in X_{\mathrm{goal}} \)
\begin{align*}
\min_{(\bm{\mathrm{x}}_{i}, \bm{\mathrm{x}}_{j}) \in \mathcal{Q}_{\mathrm{ATD}^{*}}^{\mathrm{E}}} \hspace*{-1.5em} \big\{ g_{\mathrm{p}}[\bm{\mathrm{x}}_{i}] + c(\bm{\mathrm{x}}_{i}, \bm{\mathrm{x}}_{j}) + \varepsilon_{\mathrm{infl}}\widehat{h}(\bm{\mathrm{x}}_{j}) \big\} \\[-0.5em] \geq \frac{\min\left\{g_{\mathrm{p}}[\bm{\mathrm{x}}_{\mathrm{goal}}], g_{\mathcal{T}}(\bm{\mathrm{x}}_{\mathrm{goal}})\right\}}{\varepsilon_{\mathrm{trunc}}}\hspace*{-0.1em}.
\end{align*}%
\vspace*{-1.2em}%

ABIT*'s search terminates if for any goal \( \bm{\mathrm{x}}_{\mathrm{goal}} \in X_{\mathrm{goal}} \)%
\vspace*{-0.3em}%
\begin{align*}
\min_{(\bm{\mathrm{x}}_{i}, \bm{\mathrm{x}}_{j}) \in \mathcal{Q}_{\mathrm{ABIT}^{*}}} \hspace*{-1.5em} \big\{ g_{\mathcal{T}}(\bm{\mathrm{x}}_{i}) + \widehat{c}(\bm{\mathrm{x}}_{i}, \bm{\mathrm{x}}_{j}) + \widehat{h}(\bm{\mathrm{x}}_{j}) \big\} \geq \frac{g_{\mathcal{T}}(\bm{\mathrm{x}}_{\mathrm{goal}})}{\varepsilon_{\mathrm{trunc}}}.
\end{align*}

This is less strict, as the heuristic, \( \widehat{c} \), is admissible, the inflation factor, \( \varepsilon_{\mathrm{infl}} \), is greater than or equal to one, and for all states, \( \bm{\mathrm{x}}_{i} \in X \), it holds that \( g_{\mathrm{p}}[\bm{\mathrm{x}}_{i}] \geq g_{\mathcal{T}}(\bm{\mathrm{x}}_{i}) \) as rewirings can only improve the cost-to-come to states. ABIT* therefore considers at least all edges that ATD* would consider.
\end{proof}

\begin{figure}[t]
  \centering
  \subfloat[\footnotesize Wall gap\label{fig:wall_gap}]{
%
%
%
\begin{tikzpicture}

\begin{axis}[%
width=0.45\columnwidth,
height=0.45\columnwidth,
at={(0,0)},
scale only axis,
xmin=-1,
xmax=1,
xtick={\empty},
xticklabels={\empty},
ymin=-1,
ymax=1,
ytick={\empty},
yticklabels={\empty},
axis background/.style={fill=white},
after end axis/.code={
  \draw [<->, black] (axis cs:-1.0, -1.1) -- node [fill=white, inner sep=1pt] {\scriptsize$2.0$} (axis cs:1.0, -1.1);
  \draw [densely dotted] (axis cs:-1.0, -1.1) -- (axis cs:-1.0, -1.0);
  \draw [densely dotted] (axis cs: 1.0, -1.1) -- (axis cs: 1.0, -1.0);
  \draw [<->, black] (axis cs:-1.1, -1.0) -- node [fill=white, inner sep=1pt, sloped] {\scriptsize$2.0$} (axis cs:-1.1, 1.0);
  \draw [densely dotted] (axis cs:-1.1, -1.0) -- (axis cs:-1.0, -1.0);
  \draw [densely dotted] (axis cs:-1.1,  1.0) -- (axis cs:-1.0,  1.0);
  \draw [<->, black] (axis cs:-0.5, 0.85) -- node [inner sep=1pt, fill=white] {\scriptsize$1.0$} (axis cs:0.5, 0.85);
  \draw [densely dotted] (axis cs:-0.5, 0.85) -- (axis cs:-0.5, -0.25);
  \draw [densely dotted] (axis cs: 0.5, 0.85) -- (axis cs: 0.5, 0.0);
  \draw [<->, black] (axis cs:-1.0, -0.25) -- node [inner sep=1pt, fill=white] {\scriptsize$0.5$} (axis cs:-0.5, -0.25);
  \draw [<->, black] (axis cs:-0.6, 1.0) -- node [left = 0.1cm, inner sep=1pt, fill=white] {\scriptsize$0.3$} (axis cs:-0.6, 0.7);
  \draw [<->, black] (axis cs:0.5, -1.0) -- node [sloped, midway, inner sep=1pt, fill=white] {\scriptsize$1.0$} (axis cs:0.5, 0.0);
  \draw [densely dotted] (axis cs:-0.6, 0.7) -- (axis cs:-0.125, 0.7);
  \draw [->, black] (axis cs:-0.6, 0.46) -- (axis cs:-0.6, 0.36);
  \draw [->, black] (axis cs:-0.6, 0.23) -- (axis cs:-0.6, 0.33);
  \node [left = 0.1cm, inner sep=1pt] at (axis cs: -0.6, 0.345) {\scriptsize$0.03$};
  \draw [densely dotted] (axis cs:-0.6, 0.37) -- (axis cs:-0.125, 0.37);
  \draw [densely dotted] (axis cs:-0.6, 0.32) -- (axis cs:-0.125, 0.32);
  \draw [->, black] (axis cs:-0.3, -0.75) -- (axis cs:-0.125, -0.75);
  \draw [->, black] (axis cs:0.3, -0.75) -- (axis cs:0.125, -0.75);
  \node [inner sep=0pt] at (axis cs: -0.42, -0.75) {\scriptsize$0.3$};
  \draw [<->, black] (axis cs:-1.0, -0.5) -- node [sloped, midway, inner sep=1pt, fill=white] {\scriptsize$0.85$} (axis cs:-0.15, -0.5);
  \node [left = 0.25em, fill=white, inner sep=0.9pt] at (axis cs:-0.5, 0) {\scriptsize\color{muired}start\vphantom{g}};
  \node [right = 0.25em, fill=white, inner sep=1pt] at (axis cs:0.5, 0) {\scriptsize\color{muigreen}goal};
  \node [circle, inner sep=1pt, fill=muired, draw=muired] at (axis cs:-0.5, 0) {};
  \node [circle, inner sep=1pt, fill=muigreen, draw=muigreen] at (axis cs:0.5, 0) {};
}]

\addplot[area legend, draw=black, fill=black, forget plot]
table[row sep=crcr] {%
x	y\\
-0.125	-1\\
0.125	-1\\
0.125	0.7\\
-0.125	0.7\\
}--cycle;

\addplot[area legend, draw=white, fill=white, forget plot]
table[row sep=crcr] {%
x	y\\
-0.13	0.33\\
0.13	0.33\\
0.13	0.36\\
-0.13	0.36\\
}--cycle;

\end{axis}
\end{tikzpicture}%
}
  \hfill
  \subfloat[\footnotesize Random rectangles\label{fig:random_world}]{
%
%
\begin{tikzpicture}

\begin{axis}[%
width=0.45\columnwidth,
height=0.45\columnwidth,
at={(0,0)},
scale only axis,
xmin=-1,
xmax=1,
xtick={\empty},
xticklabels={\empty},
ymin=-1,
ymax=1,
ytick={\empty},
yticklabels={\empty},
axis background/.style={fill=white},
after end axis/.code={
  \draw [<->, black] (axis cs:-1.0, -1.1) -- node [fill=white, inner sep=1pt] {\scriptsize$2.0$} (axis cs:1.0, -1.1);
  \draw [densely dotted] (axis cs:-1.0, -1.1) -- (axis cs:-1.0, -1.0);
  \draw [densely dotted] (axis cs: 1.0, -1.1) -- (axis cs: 1.0, -1.0);
  \draw [<->, black] (axis cs:-1.1, -1.0) -- node [fill=white, inner sep=1pt, sloped] {\scriptsize$2.0$} (axis cs:-1.1, 1.0);
  \draw [densely dotted] (axis cs:-1.1, -1.0) -- (axis cs:-1.0, -1.0);
  \draw [densely dotted] (axis cs:-1.1,  1.0) -- (axis cs:-1.0,  1.0);
  \draw [<->, black] (axis cs:-0.5, 0.175) -- node [inner sep=1pt, fill=white] {\scriptsize$1.0$} (axis cs:0.5, 0.175);
  \draw [densely dotted] (axis cs:-0.5, 0.175) -- (axis cs:-0.5, 0);
  \draw [densely dotted] (axis cs: 0.5, 0.175) -- (axis cs: 0.5, 0);
  \draw [densely dotted] (axis cs:-0.5, 0.0) -- (axis cs:-0.5, -0.25);
  \draw [<->, black] (axis cs:-1.0, -0.25) -- node [inner sep=1pt, fill=white] {\scriptsize$0.5$} (axis cs:-0.5, -0.25);
  \draw [<->, black] (axis cs:0.5, -1.0) -- node [sloped, midway, inner sep=1pt, fill=white] {\scriptsize$1.0$} (axis cs:0.5, 0.0);
  \node [left = 0.25em, fill=white, inner sep=0.9pt] at (axis cs:-0.5, 0) {\scriptsize\color{muired}start\vphantom{g}};
  \node [right = 0.25em, fill=white, inner sep=1pt] at (axis cs:0.5, 0) {\scriptsize\color{muigreen}goal};
  \node [circle, inner sep=1pt, fill=muired, draw=muired] at (axis cs:-0.5, 0) {};
  \node [circle, inner sep=1pt, fill=muigreen, draw=muigreen] at (axis cs:0.5, 0) {};
}]

\addplot[area legend, draw=none, fill=black]
table[row sep=crcr] {%
x	y\\
-0.0666667	-0.0666667\\
0.0666667	-0.0666667\\
0.0666667	0.0666667\\
-0.0666667	0.0666667\\
}--cycle;

\addplot[area legend, draw=none, fill=black]
table[row sep=crcr] {%
x	y\\
0.907509	0.169584\\
1.23599	0.169584\\
1.23599	0.533561\\
0.907509	0.533561\\
}--cycle;

\addplot[area legend, draw=none, fill=black]
table[row sep=crcr] {%
x	y\\
0.354489	0.554785\\
0.566012	0.554785\\
0.566012	0.898923\\
0.354489	0.898923\\
}--cycle;

\addplot[area legend, draw=none, fill=black]
table[row sep=crcr] {%
x	y\\
0.698115	-0.269408\\
0.973625	-0.269408\\
0.973625	0.110868\\
0.698115	0.110868\\
}--cycle;

\addplot[area legend, draw=none, fill=black]
table[row sep=crcr] {%
x	y\\
0.551255	-0.967855\\
0.906014	-0.967855\\
0.906014	-0.762826\\
0.551255	-0.762826\\
}--cycle;

\addplot[area legend, draw=none, fill=black]
table[row sep=crcr] {%
x	y\\
-0.910007	-0.133001\\
-0.660081	-0.133001\\
-0.660081	0.250652\\
-0.910007	0.250652\\
}--cycle;

\addplot[area legend, draw=none, fill=black]
table[row sep=crcr] {%
x	y\\
0.999466	0.0334399\\
1.18864	0.0334399\\
1.18864	0.421798\\
0.999466	0.421798\\
}--cycle;

\addplot[area legend, draw=none, fill=black]
table[row sep=crcr] {%
x	y\\
-0.57003	0.749099\\
-0.303572	0.749099\\
-0.303572	1.12207\\
-0.57003	1.12207\\
}--cycle;

\addplot[area legend, draw=none, fill=black]
table[row sep=crcr] {%
x	y\\
0.110128	0.194098\\
0.4473	0.194098\\
0.4473	0.571002\\
0.110128	0.571002\\
}--cycle;

\addplot[area legend, draw=none, fill=black]
table[row sep=crcr] {%
x	y\\
0.890696	-0.216661\\
1.10821	-0.216661\\
1.10821	0.0759259\\
0.890696	0.0759259\\
}--cycle;

\addplot[area legend, draw=none, fill=black]
table[row sep=crcr] {%
x	y\\
0.510163	-0.208722\\
0.843747	-0.208722\\
0.843747	0.132913\\
0.510163	0.132913\\
}--cycle;

\addplot[area legend, draw=none, fill=black]
table[row sep=crcr] {%
x	y\\
-0.175145	-0.674524\\
0.0470772	-0.674524\\
0.0470772	-0.394354\\
-0.175145	-0.394354\\
}--cycle;

\addplot[area legend, draw=none, fill=black]
table[row sep=crcr] {%
x	y\\
0.00283012	-0.114439\\
0.283953	-0.114439\\
0.283953	0.154221\\
0.00283012	0.154221\\
}--cycle;

\addplot[area legend, draw=none, fill=black]
table[row sep=crcr] {%
x	y\\
-0.542896	-0.584992\\
-0.294606	-0.584992\\
-0.294606	-0.377052\\
-0.542896	-0.377052\\
}--cycle;

\addplot[area legend, draw=none, fill=black]
table[row sep=crcr] {%
x	y\\
-0.436511	0.987745\\
-0.275452	0.987745\\
-0.275452	1.17374\\
-0.436511	1.17374\\
}--cycle;

\addplot[area legend, draw=none, fill=black]
table[row sep=crcr] {%
x	y\\
-0.116224	0.360012\\
0.237705	0.360012\\
0.237705	0.819933\\
-0.116224	0.819933\\
}--cycle;

\addplot[area legend, draw=none, fill=black]
table[row sep=crcr] {%
x	y\\
0.436776	0.885657\\
0.743594	0.885657\\
0.743594	1.14053\\
0.436776	1.14053\\
}--cycle;

\addplot[area legend, draw=none, fill=black]
table[row sep=crcr] {%
x	y\\
0.0200017	-0.908888\\
0.436416	-0.908888\\
0.436416	-0.451075\\
0.0200017	-0.451075\\
}--cycle;

\addplot[area legend, draw=none, fill=black]
table[row sep=crcr] {%
x	y\\
-0.448421	-0.544424\\
-0.122771	-0.544424\\
-0.122771	-0.100025\\
-0.448421	-0.100025\\
}--cycle;

\addplot[area legend, draw=none, fill=black]
table[row sep=crcr] {%
x	y\\
-0.105147	0.385477\\
0.258258	0.385477\\
0.258258	0.775053\\
-0.105147	0.775053\\
}--cycle;

\addplot[area legend, draw=none, fill=black]
table[row sep=crcr] {%
x	y\\
-0.48419	0.706303\\
-0.111851	0.706303\\
-0.111851	1.07859\\
-0.48419	1.07859\\
}--cycle;

\end{axis}
\end{tikzpicture}%
}
  \caption{\footnotesize A 2D illustration of the simulated planning problems used in Section~\ref{sec:experiments}. The state space, \( X \subset \mathbb{R}^{n} \), is bounded by a hypercube of width two for both problems. Ten different instantiations of the random rectangles experiment were tested. The results are presented in Fig.~\protect\ref{fig:experimental_results}.}%
\label{fig:experiment_illustrations}
\end{figure}
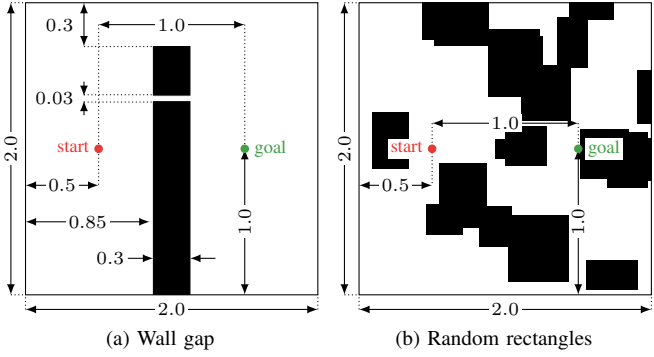

\section{Experimental Results}%
\label{sec:experiments}

ABIT* was compared against the Open Motion Planning Library (OMPL)~\citep{sucan2012} versions of RRT-Connect, RRT*, RRT$^{\#}$, LBT-RRT, and BIT* on simulated problems in \( \mathbb{R}^{4} \) and \( \mathbb{R}^{8} \) (Fig.~\ref{fig:experiment_illustrations})\footnote{The performances were measured with OMPL v1.4.1 on a laptop with 16~GB of RAM and an Intel i7-4910MQ processor running Ubuntu 18.04.}. The objective for the almost-surely asymptotically optimal planners was to minimize path length. The RGG constant \( \eta \) was set to 1.1 for all planners. LBT-RRT used the default value of 0.4 as the approximation factor. RRT$^{\#}$ sampled the entire state space. RRT-based algorithms used a goal bias of 5\% and maximum edge lengths of 0.5 and 1.25 in \( \mathbb{R}^{4} \) and \( \mathbb{R}^{8} \), respectively. BIT* and ABIT* sampled 100 states per batch regardless of the state space dimension, had graph pruning turned off, and used Euclidean distance as a heuristic. ABIT* was configured to search each RGG twice. First with a highly inflated heuristic, \( \varepsilon_{\mathrm{infl}} = 10^{6} \), and then again with a lower factor, \( \varepsilon_{\mathrm{infl}} = 1 + \nicefrac{10}{q} \). A single truncation factor, \( \varepsilon_{\mathrm{trunc}} = 1 + \nicefrac{5}{q} \) was used for all searches. All parameters were tuned to optimize planner performance on test problems.

\subsection{Experimental Problems}%
\label{sec:experimental-problems}

The planners were tested on two problems in \( \mathbb{R}^{4} \) and \( \mathbb{R}^{8} \). The first consisted of a wall with a narrow gap such that valid paths can only be in one of two homotopy classes (Fig.~\ref{fig:wall_gap}). Each planner was run 100 times for one second with different random seeds. Figures~3a and 3d show the achieved success rates and median path lengths of all tested planners.

The second consisted of axis-aligned hyperrectangles of random widths placed randomly in the state space (e.g., Fig.~\ref{fig:random_world}). Ten different random problems were generated for each state space dimension and planners were run 100 times on each instantiation. The runtime was limited to one and 40 seconds for problems in \( \mathbb{R}^{4} \) and \( \mathbb{R}^{8} \), respectively. Figures~3b, 3c, 3e, and 3f show the achieved success rates and median path costs of all tested planners for the two problems that resulted in the best and worst performances of ABIT*, as defined by its initial solution time relative to RRT-Connect.

\subsection{Planning for Axel}%
\label{sec:planning-for-axel}

The benefits of ABIT*'s advanced graph-search techniques were also demonstrated on real-world robotic planning problems during a week-long NASA/JPL-Caltech field test in the Mojave Desert with the Axel Rover System (Fig.~\ref{fig:teaser}). Axel is a tethered robotic platform designed for near-vertical surfaces and other challenging or unstable terrain. The complexity of the terrain and its interaction with the tether make for challenging planning problems because state evaluations are computationally expensive. ABIT* typically found initial solutions to these problems in under two seconds. This allowed it to spend the remaining computational time to improve this solution by repairing its search and increasing the density of its approximation. This resulted in 95.12\% autonomy by distance, despite the challenging terrain.

\begin{figure*}[t]
  \centering
  \input{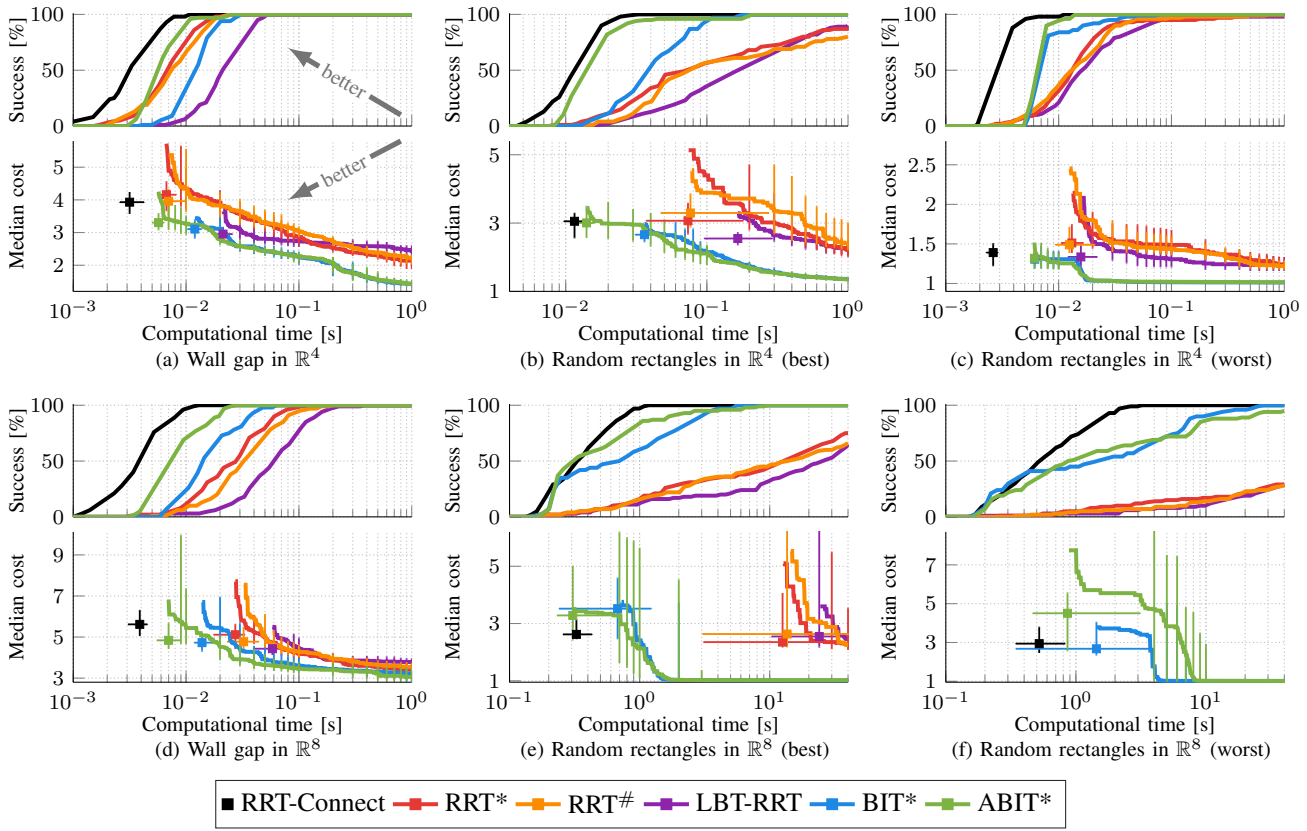}
  \caption{\footnotesize Results from the experiments described in Section~\protect\ref{sec:experimental-problems}. The results from the wall gap experiment are shown in the plots (a) and (d) for \( \mathbb{R}^{4} \) and \( \mathbb{R}^{8} \), respectively. The plots (b) and (c) show the best and worst instances of ten random rectangles experiments in \( \mathbb{R}^{4} \). The plots (e) and (f) show the best and worst instances in \( \mathbb{R}^{8} \). The squares in the cost plots show the median initial costs and times while the lines show the median cost over time for almost-surely asymptotically optimal planners (unsuccessful runs were taken as infinite costs). The error bars show a nonparametric 99\% confidence interval on the solution cost and time. Note that in plot (f) less than 50 trials of LBT-RRT, RRT* and RRT$^{\#}$ succeeded, so their median cost is infinity.\vspace{-2em}}%
\label{fig:experimental_results}
\end{figure*}

\section{Discussion \& Conclusion}%
\label{sec:discussion-and-conclusion}

ABIT* demonstrates that the perspective of separate approximation and search in single-query almost-surely asymptotically optimal sampling-based planning can be used to design algorithms with improved anytime performance. Figure~\ref{fig:experimental_results} shows that ABIT* outperforms other single-query, almost-surely asymptotically optimal planners by finding initial solutions quickly and converging to an optimal solution in an anytime manner without wasting computational effort. The only tested planner that finds initial solutions faster than ABIT* is RRT-Connect, which is not an almost-surely asymptotically optimal algorithm and cannot improve its initial solution when given more computational time.

ABIT* relies on admissible heuristic estimates of edge-costs between states and the cost-to-go from states to a goal. If no heuristics are available, ABIT* can be run with the trivial heuristic, i.e., \( \forall \bm{\mathrm{x}}_{i}, \bm{\mathrm{x}}_{j} \in X \), \( \widehat{h}(\bm{\mathrm{x}}_{i}) \equiv \widehat{c}(\bm{\mathrm{x}}_{i}, \bm{\mathrm{x}}_{j}) \equiv 0 \). An asymmetric bidirectional search could alternatively be used to simultaneously estimate and exploit a problem-specific heuristic, as in Adaptively Informed Trees (AIT*)~\citep{strub2020}.

If ABIT* is run with unit inflation and truncation factors, it can be viewed as a simplified but equally performant version of BIT* that cascades rewirings. ABIT* uses a single edge queue instead of BIT*'s dual vertex and edge queues and avoids repeated collision checks by caching checked edges in an object-oriented manner instead of labelling states \emph{old} or \emph{new} as in BIT*. This clarifies the conceptual ideas behind these algorithms and simplifies their implementation without adding any practically significant computational costs.

This improved implementation allows ABIT* to balance exploiting its current approximation of the state space with exploring the relevant regions of the state space. This is achieved using advanced graph-search techniques similar to anytime repairing and truncated search algorithms.

An inflated heuristic biases ABIT*'s search towards the goal and finds initial solutions quickly. Truncating the search once a sufficient bound on the solution quality of the current solution is achieved avoids wasting computational effort fully exploiting an approximation that will change. Flexible update policies of the inflation and truncation factors ensure that ABIT* can leverage high and low inflation and truncation depending on the accuracy of its approximation. ABIT* is not very sensitive to the exact form of these policies. Results comparable to the ones presented in this paper are achieved whenever the initial search is conducted with a very high inflation factor and both factors asymptotically tend to one as the number of sampled states approaches infinity.

ABIT* also shows the benefits of using advanced graph-search techniques in sampling-based planning on real-world path planning problems posed by Axel, a NASA/JPL-Caltech rover specialized for navigation on challenging terrain. 

Information on the OMPL implementation of ABIT* is available at \texttt{\small\href{https://robotic-esp.com/code/}{https://robotic-esp.com/code/}}.

\section{Acknowledgements}%
\label{sec:acknowledgements}

This research was funded by UK Research and Innovation and EPSRC through the Robotics and Artificial Intelligence for Nuclear (RAIN) research hub [EP/R026084/1]. We want to thank the entire NASA/JPL-Caltech Robotic Surface Mobility Group for their support on the Axel Rover System and extend special thanks to Travis, Jacob, Issa, and Mike for their help with integrating ABIT* onto Axel.

\bibliographystyle{IEEEtran}
\bibliography{abitstar_bibliography}

\end{document}